%% file: Paper.tex
\documentclass{article}

\PassOptionsToPackage{numbers, compress}{natbib}
\usepackage[final]{nips_2018}

\usepackage[utf8]{inputenc} 
\usepackage[T1]{fontenc}    
\usepackage{hyperref}       
\usepackage{url}            
\usepackage{booktabs}       
\usepackage{amsfonts}       
\usepackage{nicefrac}       
\usepackage{microtype}      
\usepackage{wrapfig}

\usepackage{graphicx}
\usepackage{subfigure}
\usepackage{booktabs} 

\usepackage{hyperref}

\usepackage[utf8]{inputenc} 
\usepackage[T1]{fontenc}    
\usepackage{hyperref}       
\usepackage{url}            
\usepackage{booktabs}       
\usepackage{amsfonts}       
\usepackage{nicefrac}       

\usepackage{algorithm}
\usepackage[noend]{algorithmic}
\usepackage{amsmath}
\usepackage{amssymb}
\usepackage{amsthm}
\usepackage{bbm}
\usepackage{bm}
\usepackage{color}
\usepackage{dsfont}
\usepackage{enumerate}
\usepackage{enumitem}
\usepackage{times}
\usepackage{thmtools, thm-restate}

\usepackage[usenames,dvipsnames]{xcolor}
\hypersetup{
  pdftex,
  pdffitwindow=true,
  pdfstartview={FitH},
  pdfnewwindow=true,
  colorlinks,
  linktocpage=true,
  linkcolor=Green,
  urlcolor=Green,
  citecolor=Green
}

\usepackage[backgroundcolor=White,textwidth=0.75in,disable]{todonotes}

\newcommand{\commentout}[1]{}
\newcommand{\junk}[1]{}

\newcommand{\E}[2]{\mathbb{E}_{#1} \! \left[#2\right]}
\newcommand{\1}[2]{\mathbbm{1}_{#1} \! \left(#2\right)}

\newtheorem*{theorem*}{Theorem}
\newtheorem{corollary}{Corollary}
\newtheorem{lemma}{Lemma}
\newtheorem*{lemma*}{Lemma}

\newtheorem{definition}{Definition}

\DeclareMathOperator*{\argmax}{arg\,max\,}

\mathchardef\mhyphen="2D

\newcommand{\rnd}[1]{\bm{#1}}
\newcommand{\omm}{{\tt OMM}}
\newcommand{\cmucb}{{\tt I\mhyphen UCB}}
\newcommand{\cmucbone}{{\tt I\mhyphen UCB1}}
\newcommand{\cmucbtwo}{{\tt I\mhyphen UCB2}}

\title{Conservative Exploration using Interleaving}
\author{
    Sumeet Katariya \\
    University of Wisconsin-Madison \\
    \texttt{katariya@wisc.edu} 
    \And 
    Branislav Kveton\thanks{This work was done while the author was at Adobe Research.} \\
    Google Research \\
    \texttt{bkveton@google.com} 
    \And 
    Zheng Wen \\
    Adobe Research \\
    \texttt{zwen@adobe.com} 
    \And 
    Vamsi K. Potluru \\
    Comcast Cable \\
    \texttt{vamsi\_potluru@cable.comcast.com} 
}

\usepackage[capitalize]{cleveref}

\begin{document}

\maketitle
\vspace{-15pt}
\begin{abstract}
In many practical problems, a learning agent may want to learn the best action in hindsight without ever taking a bad action, which is significantly worse than the default production action. In general, this is impossible because the agent has to explore unknown actions, some of which can be bad, to learn better actions. However, when the actions are combinatorial, this may be possible if the unknown action can be evaluated by interleaving it with the production action. We formalize this concept as learning in stochastic combinatorial semi-bandits with exchangeable actions. We design efficient learning algorithms for this problem, bound their $n$-step regret, and evaluate them on both synthetic and real-world problems. Our real-world experiments show that our algorithms can learn to recommend $K$ most attractive movies without ever violating a strict production constraint, both overall and subject to a diversity constraint.
\vspace{-3pt}
\end{abstract}

\input{Introduction}

\input{Setting}

\input{Algorithm}

\input{Analysis}

\input{Experiments}

\input{RelatedWork}

\input{Conclusions}

\newpage
\bibliographystyle{plainnat}
\bibliography{References}

\input{Appendix}

\end{document}

%% file: Introduction.tex

\section{Introduction}
\label{sec:introduction}


Recommender systems are an integral component of many industries, with applications in content personalization, advertising, and landing page design \cite{resnick1997recommender, adomavicius2015context, broder2008computational}. Multi-armed bandit algorithms provide adaptive techniques for content recommendation, and although theoretically well-understood, they have not been widely adopted in production systems \citep{cremonesi2011looking,schnabel2018short}. This is primarily due to concerns that the output of the bandit algorithm can be sub-optimal or even disastrous, especially when the algorithm explores sub-optimal arms. To address this issue, most industries have a static recommendation engine in production that has been well-optimized and tested over many years, and a promising new policy is often evaluated using A/B testing \cite{siroker2013b} by allocating a small percentage $\alpha$ of the traffic to the new policy. When the utilities of actions are independent, this is a reasonable solution that allows the new policy to explore non-aggressively. 

Many recommendation problems, however, involve \emph{structured actions}, such as ranked lists of items (movies, products, etc.). In such actions, the total utility of the action can be decomposed into the utilities of its individual items. Therefore, it is conceivable that the new policy can be evaluated in a controlled and principled fashion by \emph{interleaving} items in the new and production actions, instead of splitting the traffic as is done in A/B testing. As a concrete example, consider the problem of recommending top-$K$ movies to a new visitor \citep{deshpande2004item}. A company may have a production policy that recommends a default set of $K$ movies that performs reasonably well, but intends to test a new algorithm that promises to learn better movies. The A/B testing method would show the new algorithm's recommendations to a visitor with probability $\alpha$. In the initial stages, the new algorithm is expected to explore a lot to learn, and may hurt engagement with the visitor who is shown a disastrous set of movies, just to learn that these movies are not good. However, an arguably better approach that does not hurt any visitor's engagement as much and gathers the same feedback on average, is to show the default well-tested movies interleaved with $\alpha$ fraction of new recommendations. A recent study by \citet{schnabel2018short} concluded that this latter approach is in fact better:
\begin{quote}
``These findings indicate that for improving recommendation systems in practice, it is preferable to mix a limited amount of exploration into every impression – as opposed to having a few impressions that do pure exploration.''
\end{quote} 


In this paper, we formalize the above idea and study the general case where actions are \emph{exchangeable}, which is a mathematical formulation of the notion of interleaving. One fairly general and important class of exchangeable actions is the set of bases of a matroid, and this is the setting we focus on in our theorems and experiments. In particular, we study learning variants of maximizing an unknown modular function on a matroid subject to a conservative constraint. 


In the recommendations problem discussed above, our conservative constraint requires that the recommendations always be above a certain baseline quality. The question we wish to answer is: \emph{what is the price of being this conservative?} In this work, we answer this question and make five contributions. First, we introduce the idea of \emph{conservative multi-armed bandits in combinatorial action spaces}, and formulate a conservative constraint that addresses the issues raised in \citet{schnabel2018short}. Existing conservative constraints for multi-armed bandit problems fail in this aspect, and hence our constraint is more appropriate for combinatorial action spaces. Second, we propose interleaving as a solution, and show how it naturally leads to the idea of \emph{exchangeable} action spaces. We precisely formulate an online learning problem - \emph{conservative interleaving bandits} - in one such space, that of matroids. Third, we present \emph{Interleaving Upper Confidence Bound ($\cmucb$)}, a computationally and sample-efficient algorithm for solving our problem. The algorithm satisfies our conservative constraint by design. Fourth, we prove gap-dependent upper bounds on its expected cumulative regret, and show that the regret scales logarithmically in the number of steps $n$, at most linearly in the number of items $L$, and at most quadratically in the number of items $K$ in any action. Finally, we evaluate $\cmucb$ on both synthetic and real-world problems. In the synthetic experiments, we validate an extra factor in our regret bounds, which is the price for being conservative. In the real-world experiments, we illustrate how to formulate and solve top-$K$ recommendation problems in our setting. To the best of our knowledge, this is the first work that studies conservatism in the context of combinatorial bandit problems.

%% file: Setting.tex

\vspace{-10pt}
\section{Setting}
\label{sec:setting}

We focus on linear reward functions and formulate our learning problem as a stochastic combinatorial semi-bandit \citep{kveton2015tight, gai2012combinatorial, chen2013combinatorial}, which we first review in \cref{sec:semi-bandit}. Stochastic combinatorial semi-bandits have been used for recommendation problems before \citep{kveton2014learning, kveton2014matroid}. In \cref{sec:exchangeability}, we motivate our notion of conservativeness, and suggest interleaving as a solution, which can be mathematically formulated using exchangeable action spaces.
Finally, in \cref{sec:conmatbandit}, we show that actions that are bases of a matroid are exchangeable, and phrase our problem using the terminology of matroids. To simplify exposition, we write all random variables in bold. We use $[K]$ to denote the set $\{1,\dots,K\}$.

\subsection{Stochastic Combinatorial Semi-Bandits}
\label{sec:semi-bandit}

A \emph{stochastic combinatorial semi-bandit} \citep{kveton2015tight, gai2012combinatorial, chen2013combinatorial} is a tuple $(E, \mathcal{B}, P)$, where $E=[L]$ is a finite set of $L$ items, $\mathcal{B} \subseteq \Pi_K(E)$ is a non-empty set of feasible subsets of $E$ of size $K$, and $P$ is a probability distribution over a unit cube $[0,1]^E$. Here $\Pi_K(E)$ is the set of all $K$-permutations of $E$. 

Let $(\rnd{w}_t)_{t = 1}^n$ be an i.i.d. sequence of $n$ weights drawn according to $P$, where $\rnd{w}_t(e)$ is the weight of item $e \in E$ at time $t$. The learning agent interacts with our problem as follows. At time $t$, it takes an action $\rnd{A}_t \in \mathcal{B}$, which is a set of items from $E$. The reward for taking the action is $f(\rnd{A}_t, \rnd{w}_t)$, where $f(A,w) = \sum_{e \in A} w(e)$ is the sum of the weights of items in $A$ in weight vector $w$. After taking action $\rnd{A}_t$, the agent observes the weight $\rnd{w}_t(e)$ for each item $e \in \rnd{A}_t$. This model of feedback is known as \emph{semi-bandit} \cite{audibert2013regret}.


The learning agent is evaluated by its \emph{expected $n$-step regret} $R(n) = \E{}{\sum_{t = 1}^n R(\rnd{A}_t, \rnd{w}_t)}$, where $R(\rnd{A}_t, \rnd{w}_t) = f(A_\ast, \rnd{w}_t) - f(\rnd{A}_t, \rnd{w}_t)$ is the \emph{instantaneous stochastic regret} of the agent at time $t$ and $A_\ast = \argmax_{A \in \mathcal{B}} f(A, \bar{w})$ is the \emph{maximum weight action in hindsight}.

\subsection{Conservativeness and Exchangeable Actions}
\label{sec:exchangeability}

The idea of controlled exploration is not new. \citet{wu2016conservative} studied conservatism in multi-armed bandits, and their learning agent is constrained to have its cumulative reward no worse than $1 - \alpha$ of that of the default action. In this sense, their conservative constraint is \emph{cumulative}. Roughly speaking, the constraint means that the learning agent can explore once in every $1 / \alpha$ steps. 

A/B testing can also be thought of as the solution to a constrained exploration problem where the constraint is \emph{instantaneous} (instead of cumulative); here the constraint requires that the actions at any time be at least $(1-\alpha)$ good as the default action \emph{in expectation}, where the expectation is taken over multiple runs of the A/B test. 

When actions are combinatorial, as in the top-$K$ movie recommendation problem in \cref{sec:introduction}, both these forms of conservatism allow the learning agent to occasionally take actions containing items that are all disastrous (for example, have very low popularity). We consider a stricter conservative constraint that explicitly forbids this possibility.

We state our conservative constraint next. Let $K$ be the number of items in any action. Let $B_0$ be the default baseline action, where $|B_0|=K$. Our constraint requires that at any time $t$, the action $\rnd{A}_t$ should be at least as good as the baseline set $B_0$, in the sense that most items in $\rnd{A}_t$ are at least as good or better than those in $B_0$. Mathematically, we require that there exists a bijection $\rho_{\rnd{A}_t,B_0}: \rnd{A}_t \rightarrow B_0$ such that 
\begin{align}
  \sum_{e \in \rnd{A}_t} \1{}{\bar{w}(e) \geq \bar{w}(\rho_{\rnd{A}_t,B_0}(e))} \geq (1 - \alpha) K
  \label{eq:conservative constraint}
\end{align}
holds with a high probability at any time $t$. That is, the items in $\rnd{A}_t$ and $B_0$ can be matched such that no more than $\alpha$ fraction of the items in $\rnd{A}_t$ has a lower expected reward than those in $B_0$. For simplicity of exposition, we only consider the special case of $\alpha = 1 / K$ in this work. We discuss the case $\alpha > 1/K$ in \cref{sec:discussion}.

Given an algorithm that explores and suggests new actions that could potentially be disastrous, a simple way to satisfy \eqref{eq:conservative constraint} is to \emph{interleave} most items from the default action with a few from the new action. 
This is possible if the set of feasible actions $\mathcal{B} \subseteq 2^E$ is \emph{exchangeable}, which we define next.

\begin{definition}
\label{defn:exchangeable}
A set $\mathcal{B} \subseteq 2^E$ is exchangeable if for any two actions $A_1,A_2 \in \mathcal{B}$, there exists a bijection $\rho_{A_1,A_2}:A_1 \rightarrow A_2$ such that 
\begin{align}
    \forall\,e \in A_1: A_1 \setminus \{e\} \cup \{\rho_{A_1,A_2}(e)\} \in \mathcal{B}\,.
    \label{eq:exchangeability definition}
\end{align}
\end{definition}
In our motivating top-$K$ movie recommendation example, $A_1$ is the default action (recommendation) and $A_2$ is the new action, and $|A_1|=|A_2|=K$. If the action space is exchangeable, we can explore all items in a new action $A_2$ over $K$ time steps by taking $K$ interleaved actions. Each interleaved action substitutes an item $e \in A_1$ with the item $\rho_{A_1,A_2}(e) \in A_2$.

\vspace{-5pt}
\subsection{Conservative Interleaving Bandits}
\label{sec:conmatbandit}
\vspace{-5pt}

In this section, we consider an important exchangeable action space, the bases of a matroid. A matroid $M$ is a pair $(E,\mathcal{B})$ where $E=[L]$ is a finite set, and $\mathcal{B}\subseteq \Pi_K(E)$ is a collection of subsets of $E$ called \emph{bases} \citep{welsh1976matroid}. $K$ is called the rank of the matroid. 

Matroids have many interesting properties \citep{oxley2006matroid};  the one that is relevant to our work is the \emph{bijective exchange lemma for matroids} \cite{brualdi1969comments}, which states that the collection $\mathcal{B}$ is exchangeable.
\begin{lemma}[Bijective Exchange Lemma]
    \label{lem:bijectiveexchange} For any two bases $B_1, B_2 \in \mathcal{B}$, there exists a bijection $\rho_{B_1,B_2}: B_1 \rightarrow B_2$ such that $(B_1 \setminus \{e\}) \cup \{\rho_{B_1,B_2}(e)\}$ is a basis for any $e \in B_1$.
\end{lemma} 

The recommendations for the top-$K$ movie problem in \cref{sec:introduction} are bases of a \emph{uniform matroid}, which is a matroid whose items $E$ are movies and whose feasible sets are all $K$-permutations of these items, i.e., $\mathcal{B} = \Pi_K(E)$. One can also enforce diversity in the recommendations by formulating actions as the feasible set of a \emph{partition matroid}, which is defined as follows.
Let $\mathcal{P}_1, \dots, \mathcal{P}_K$ be a partition of $[L]$. The feasible set of the partition matroid is
$\mathcal{B} = \{A \in \Pi_K([L]): A(1) \in \mathcal{P}_1, \dots, A(K) \in \mathcal{P}_K\}$.
The members of the partition in this case correspond to the movie categories, and the partition matroid ensures that the recommended movies contain a movie from every category.
In both the above matroids, $\rho_{A,B}$ maps the $k$-th item in $A$, $A(k)$, to the $k$-th item in $B$, $B(k)$. We study both of these examples in our experiments (\cref{sec:experiments}). In addition to these examples, many important combinatorial optimization problems can be formulated as optimization on a matroid.

We formulate our learning problem using the terminology of matroids as a conservative interleaving bandit. A \emph{conservative interleaving bandit} is a tuple $(E, \mathcal{B}, P, B_0, \alpha)$, where $E=[L]$ is a set of items, $\mathcal{B} \subseteq \Pi_K(E)$ is the collection of bases, $P$ is a probability distribution over the weights $\rnd{w} \in \mathbb{R}^L$ of items $E$, the \emph{input baseline set} $B_0\in \mathcal{B}$ is a basis, and $\alpha \in [0,1]$ is a tolerance parameter. 

We assume that the matroid $(E, \mathcal{B})$, input baseline set $B_0$, and tolerance $\alpha$ are known and that the distribution $P$ is unknown.  Without loss of generality, we assume that the support of $P$ is a bounded subset of $[0,1]^L$. We denote the expected weights of items by $\bar{w} = \mathbb{E}[\rnd{w}]$.

%% file: Algorithm.tex

\vspace{-10pt}
\section{Algorithm}
\label{sec:algorithm}
\vspace{-5pt}
Learning in conservative interleaving bandits is non-trivial. For instance, one cannot simply construct exploratory sets $\rnd{D}_t$ using a non-conservative matroid bandit algorithm \citep{kveton2014matroid, talebi2016optimal}, and then take actions $\rnd{A}_t$ containing $(1-\alpha)$ fraction of items from the initial baseline set $B_0$ and the remaining items from $\rnd{D}_t$. If the set $B_0$ contains sub-optimal items, the regret of this policy is linear since its actions never converge to the optimal action $A^\ast$.

In this section, we introduce our \emph{Interleaving Upper Confidence Bound} ($\cmucb$) algorithm which achieves sub-linear regret by maintaining a baseline set $\rnd{B}_t$ which continuously improves over the initial baseline set $B_0$ with high probability. We present two variants of the algorithm: one where the agent knows the expected rewards of the input baseline set $\{\bar{w}(e): e \in B_0\}$, which we
call $\cmucbone$; and one where the learner does not know them, which we call $\cmucbtwo$. The expected rewards of items in $B_0$  may be known in practice, for instance if the baseline policy has been deployed for a while. We refer to the common aspects of both algorithms as $\cmucb$.

The pseudocode of both algorithms is in \cref{alg:main}. We highlight differences in comments. Recall that $K$ is the rank of the matroid, or equivalently the number of items in any action. $\cmucb$ operates in rounds, which are indexed by $t$, and takes $K$ actions in each round. We assume that $\cmucb$ has access to an oracle {\sc MaxBasis} that takes in a matroid and a vector of weights $w \in [L]$, and returns the maximum weight basis with respect to the weights $w$. {\sc MaxBasis} is a greedy algorithm for matroids and hence can run in $O(L \log L)$ time \citep{edmonds1971matroids}. 

Each round has three stages. In the first stage (lines $5$--$8$), $\cmucb$ computes \emph{upper confidence bounds (UCBs)} $\rnd{U}_t \in (\mathbb{R}^+)^E$ and \emph{lower confidence bounds (LCBs)} $\rnd{L}_t \in (\mathbb{R}^+)^E$ on the rewards of all items. For any item $e \in E$, let
\begin{align}
    \rnd{U}_t(e)  = \hat{\rnd{w}}_{\rnd{T}_{t-1}(e)}(e) + c_{n, \rnd{T}_{t-1}(e)}, \qquad
    \rnd{L}_t(e)  = \max \{\hat{\rnd{w}}_{\rnd{T}_{t-1}(e)}(e) - c_{n, \rnd{T}_{t-1}(e)}, 0\}
  \label{eq:ucblcb}
\end{align}
where $\hat{\rnd{w}}_{s}(e)$ is the average of $s$ observed weights of item $e$, $\rnd{T}_{t}(e)$ is the number of times item $e$ has been observed in $t$ steps, and
\begin{align}
  c_{n,s} = \sqrt{1.5 \log (n) / s}
  \label{eq:ucb1}
\end{align}
is the radius of a confidence interval around $\hat{\rnd{w}}_{s}(e)$ such that $\bar{w}(e) \in [\hat{\rnd{w}}_{s}(e)-c_{n,s}, \hat{\rnd{w}}_{s}(e)+c_{n,s}]$ holds with a high probability. We adopt UCB1 confidence intervals \cite{auer2002finite} to simplify analysis, but it is possible to use tighter ${\tt KL\mhyphen UCB}$ confidence intervals \cite{garivier2011kl}.

\begin{algorithm}[t!]
    \caption{$\cmucb$ for conservative interleaving bandits.}
	\label{alg:main}
	\begin{algorithmic}[1]
        \STATE \textbf{Input:} Set of items $E$, Collection of exchangeable actions $\mathcal{B}$, baseline action $B_0 \in \mathcal{B}$
        \STATE 
        \STATE Observe $\rnd{w}_0 \sim P$, $\forall \, e \in E:$ $\rnd{T}_0(e) \gets 1, \hat{\rnd{w}}_0(e) \gets \rnd{w}_0(e)$ \hfill // Initialization
        \STATE
        \FOR{$t = 1, 2, \dots$}
            \FOR[// Compute UCBs and LCBs]{$e \in E$}
                \STATE $\rnd{U}_t(e) = \hat{\rnd{w}}_{\rnd{T}_{t-1}(e)}(e) + c_{n,\rnd{T}_{t-1}(e)}$
                \STATE $\rnd{L}_t(e) = \max \{\hat{\rnd{w}}_{\rnd{T}_{t-1}(e)}(e) - c_{n,\rnd{T}_{t-1}(e)}, 0\}$
            \ENDFOR
            \STATE
            \STATE $\rnd{D}_t \gets ${\sc MaxBasis}$( (E,\mathcal{B}), \, \rnd{U}_t)$ \hfill // Compute decision set
            \STATE
            \FOR[// Compute baseline set]{$e \in B_0$}
                \IF[// $\cmucb1$]{$\bar{w}(e)$ is known} 
                    \STATE $\rnd{v}_t(e) \gets \bar{w}(e)$
                \ELSE[// $\cmucb2$]
                    \STATE $\rnd{v}_t(e) \gets \rnd{U}_t(e)$
                \ENDIF
            \ENDFOR
            \FOR{$e \in E \setminus B_0$}
                \STATE $\rnd{v}_t(e) \gets \rnd{L}_{t}(e)$
            \ENDFOR
            \STATE $\rnd{B}_t \gets ${\sc MaxBasis}$( (E, \mathcal{B}), \, \rnd{v}_t)$
            \STATE
            \STATE // Take $K$ combined actions of $\rnd{D}_t$ and $\rnd{B}_t$
            \STATE Let $\rnd{\rho}_t:\rnd{B}_t \rightarrow \rnd{D}_t$ be the bijection in \cref{lem:bijectiveexchange}
            \FOR{$e \in \rnd{B}_t$}
            \STATE Take action $\rnd{A}_t = \rnd{B}_t \setminus \{e\}\cup \{\rnd{\rho}_t(e)\}$
                \STATE Observe $\{\rnd{w}_t(e): e \in \rnd{A}_t\}$, where $\rnd{w}_t \sim P$
                \STATE Update statistics $\hat{\rnd{w}}$ and $\rnd{T}$
            \ENDFOR
        \ENDFOR
    \end{algorithmic}
\end{algorithm} 

In line $10$, $\cmucb$ chooses a \emph{decision set} $\rnd{D}_t$
which is the maximum weight basis with respect to $\rnd{U}_t$, an optimistic estimate of $\bar{w}$. The same approach was used in \emph{Optimistic Matroid Maximization} (OMM) of \citet{kveton2014matroid}. However, unlike OMM, we cannot take action $\rnd{D}_t$ because this action may not satisfy our conservative constraint in \eqref{eq:conservative constraint}.

In the second stage (lines $12$--$19$), $\cmucb$ computes a \emph{baseline set} $\rnd{B}_t$
which improves over the input baseline set $B_0$ in each item with a high probability. The set $\rnd{B}_t$ is the maximum weight basis with respect to weights $\rnd{v}_t$, which are chosen as follows. For items $e \in B_0$, we set $\rnd{v}_t(e) = \bar{w}(e)$ if $\bar{w}(e)$ is known, and $\rnd{v}_t(e) = \rnd{U}_t(e)$ if it is not. For items $e \in E \setminus B_0$, we set $\rnd{v}_t(e) = \rnd{L}_t(e)$. This setting guarantees that an item $e \in E \setminus B_0$ is selected over an item $e' \in B_0$ only if its expected reward is higher than that of item $e'$ with a high probability.

In the last stage (lines $22$--$26$), $\cmucb$ takes $K$ combined actions of $\rnd{D}_t$ and $\rnd{B}_t$, which are guaranteed to be bases by \cref{lem:bijectiveexchange}.

%
Let $\rnd{\rho}_t: \rnd{B}_t \rightarrow \rnd{D}_t$ be the bijection in \cref{lem:bijectiveexchange}. Then in round $t$, $\cmucb$ takes actions $\rnd{A}_t = \rnd{B}_t \setminus \{e\} \cup \{\rnd{\rho}_t(e)\}$ for all $e \in \rnd{B}_t$. Since $\rnd{A}_t$ contains at
least $K - 1$ baseline items, all of which improve over $B_0$ with a high probability, the conservative constraint in \eqref{eq:conservative constraint} is satisfied.

%% file: Analysis.tex

\vspace{-5pt}
\section{Analysis}
\label{sec:analysis}
\vspace{-5pt}

This section is organized as follows. We have three subsections. In \cref{sec:analysis 1}, we state theorems about the conservativeness of $\cmucbone$ and bound its regret. In \cref{sec:analysis 2}, we state analogous theorems for $\cmucbtwo$. In \cref{sec:discussion}, we discuss our theoretical results. We only explain the main ideas in the proofs. The details can be found in Appendix.

We use the following conventions in our analysis. Without loss of generality, we assume that items in $E$ are sorted such that $\bar{w}(1) \ge \dots \ge \bar{w}(L)$. The decision set at time $t$ is denoted by $\rnd{D}_t$, the baseline set at time $t$ is denoted by $\rnd{B}_t$, and the optimal set is denoted by $A^\ast$.
Recall that $A^\ast$, $\rnd{D}_t$, and $\rnd{B}_t$ are bases. Let $\rnd{\pi}_t: A^\ast \rightarrow \rnd{D}_t$ and $\rnd{\sigma}_t: \rnd{D}_t \rightarrow \rnd{B}_t$ be the bijections guaranteed by \cref{lem:bijectiveexchange}.  
For any item $e$ and item $e'$ such that $\bar{w}(e') > \bar{w}(e)$, we define the \emph{gap}
$\Delta_{e,e'} = \bar{w}(e') - \bar{w}(e)$.

\subsection{$\cmucb1$: Known Baseline Means}
\label{sec:analysis 1}
We first prove that $\cmucb1$ is conservative in \cref{thm:conservative 1}. Then we prove a gap-dependent upper bound on its regret in \cref{thm:regret bound cmucb1}.
\begin{restatable}{theorem}{TheoremCmucbOneAccuracy}
    \label{thm:conservative 1} $\cmucbone$ satisfies \eqref{eq:conservative constraint} for $\alpha = 1 / K$ at all time steps $t \in [n]$ with probability of at least $1 - 2 L / (K n)$.  
\end{restatable}

The regret upper bound of $\cmucb1$ involves two kinds of gaps. For every suboptimal item $e$, we define its minimum gap from the closest optimal item $e^\ast$ whose mean is higher than that of $e$ as
\begin{align}
    \textstyle
    \Delta_{e,\min} = \min_{e^\ast \in A^\ast:\Delta_{e,e^\ast}>0} \Delta_{e,e^\ast}. \label{eq:Deltaemin}
\end{align}
This gap is standard in matroid bandits \citep{kveton2014matroid}.

For any optimal item $e^\ast$, we define its minimum gap from the closest sub-optimal item $e$ whose mean is lower than that of $e^\ast$ as 
\vspace{-10pt}
\begin{align}
    \textstyle
    \Delta^\ast_{e^\ast,\min} = \min_{e \in E\setminus A^\ast: \Delta_{e,e^\ast}>0} \Delta_{e,e^\ast}. \label{eq:Deltaestarmin}
\end{align}

\begin{restatable}[Regret of $\cmucb1$]{theorem}{TheoremRegretCmucbOne}
\label{thm:regret bound cmucb1}

    The expected $n$-step regret of $\cmucb1$ is bounded as 
\begin{align*}
    (K-1)\left( 12\sum_{e^\ast \in A^\ast} \frac{1}{\Delta^\ast_{e^\ast,\min}} + 24\sum_{e \in E\setminus A^\ast} \frac{1}{\Delta_{e,\min}}  \right)\log n + 12\sum_{e \in E\setminus A^\ast} \frac{1}{\Delta_{e,\min}} \log n + c, 
\end{align*}
where 
$\Delta_{e,\min}$ and $\Delta^\ast_{e^\ast,\min}$ are defined in \eqref{eq:Deltaemin} and \eqref{eq:Deltaestarmin} respectively, and
$c = O(KL\sqrt{\log n})$.
\end{restatable}
\begin{proof}
    The standard UCB counting argument does not work because the baseline set is selected using lower confidence bounds (LCBs). Instead, we use the exchangeability property of matroids (\cref{lem:bijectiveexchange}) to match every item in the baseline set with an item in the decision set. Since the baseline set is selected using LCBs, the LCBs of the baseline items must be higher than those of the corresponding decision set items. We use this to bound the regret of the baseline set by the confidence intervals of the decision set items (\cref{lem:optimal baseline}). We then consider two cases depending on whether an item from the decision set is optimal or not. The first case leads to the first term containing the gap $\Delta^\ast_{e^\ast,\min}$, and the second case gives rise to the second term containing the gap $\Delta_{e,\min}$.
\end{proof}

\subsection{$\cmucb2$: Unknown Baseline Means}
\label{sec:analysis 2}

We first prove that $\cmucb2$ is conservative in \cref{thm:conservative 2}. Then we prove a gap-dependent upper bound on its regret in \cref{thm:regret bound cmucb2}.
\begin{restatable}{theorem}{TheoremCmucbTwoAccuracy}
    \label{thm:conservative 2} $\cmucbtwo$ satisfies \eqref{eq:conservative constraint} for $\alpha = 1 / K$ at all time steps $t \in [n]$ with probability of at least $1 - 2 L / (K n)$. 
\end{restatable}

The upper bound on the regret of $\cmucb2$ requires a third kind of gap in addition to those defined in \eqref{eq:Deltaemin} and \eqref{eq:Deltaestarmin}. For items $e' \in B_0$, we define its minimum gap from the closest item $e$ whose mean is higher than that of $e'$ as 
\vspace{-10pt}
\begin{align}
    \textstyle
    \Delta'_{e',\min} = \min_{e \in E \setminus B_0:\bar{w}(e)>\bar{w}(e')} \Delta_{e',e}. \label{eq:Deltaetildemin}
\end{align} 
\vspace{-15pt}
\begin{restatable}[Regret of $\cmucb2$]{theorem}{TheoremRegretCmucbTwo}
    \label{thm:regret bound cmucb2}
    The expected $n$-step regret of $\cmucb2$ is bounded as 
 \begin{align*}
     (K-1)\left(48\hspace{-5pt}\sum_{e^\ast \in A^\ast}\hspace{-5pt} \frac{1}{\Delta^\ast_{e^\ast, \min}}
         +36\hspace{-8pt}\sum_{e \in E\setminus A^\ast}\hspace{-5pt} \frac{1}{\Delta_{e,\min}}
     +48\hspace{-5pt}\sum_{e'\in B_0}\hspace{-5pt} \frac{1}{\Delta'_{e',\min}} 
 \right) \log n   + 24\hspace{-8pt}\sum_{e \in E\setminus A^\ast}\hspace{-5pt} \frac{1}{\Delta_{e,\min}} \log n + c,
 \end{align*}
where 
$\Delta_{e,\min}$, $\Delta^\ast_{e^\ast,\min}$ and $\Delta'_{e',\min}$ are defined in \eqref{eq:Deltaemin}, \eqref{eq:Deltaestarmin}, and \eqref{eq:Deltaetildemin} respectively, and $c = O(KL\sqrt{\log n})$.
\end{restatable}
\begin{proof}
The first two terms in the regret upper bound arise similarly to \cref{thm:regret bound cmucb1}. The additional complexity in the analysis of $\cmucb2$ stems from the fact that items in the initial baseline set $B_0$ are selected in $\rnd{B}_t$ using their UCBs, while other items are selected using their LCBs. Because of this, the regret due to items in $\rnd{B}_t \cap B_0$ is bounded using the sum of the confidence intervals of items in $\rnd{B}_t \cap B_0$ and those of the corresponding items in $\rnd{D}_t$ (\cref{lem:cmucb2}). We then consider two cases depending on whether the confidence intervals of the items in $B_0$ are smaller or larger than those of their corresponding decision set items. The latter case gives rise to the third gap term $\Delta'_{e',\min}$.
\end{proof}

\vspace{-5pt}
\subsection{Discussion}
\label{sec:discussion}
\vspace{-5pt}

We note three points. First, the regret bound of $\cmucb$ contains an extra $(K-1)$ factor as compared to the bound of non-conservative matroid bandit algorithms \citep{kveton2014matroid, talebi2016optimal}. This is because $\cmucb$ explores a new action in $K$ steps that non-conservative algorithms can explore in a single step. Note that we set $\alpha=1/K$ in our conservative constraint \eqref{eq:conservative constraint}. If the action space allows exchanging multiple items in \cref{eq:exchangeability definition}, our algorithm can be generalized to any $\alpha=m/K$ for $m \in [K]$ by interleaving multiple items simultaneously in lines $23$-$26$. It is clear from our proofs that the regret bound of this algorithm for general $\alpha$ will contain an extra factor of $K(1-\alpha)$. \emph{This is the price we pay for conservativism.} As $\alpha$ approaches $1$, this extra factor disappears and our regret upper bound matches existing regret bounds of non-conservative matroid algorithms \citep{kveton2014matroid, talebi2016optimal}.

Second, by using the standard technique of decomposing the gaps into those that are larger than $\varepsilon$ and smaller than $\varepsilon$, one can show that the gap-free regret bound is $O(K\sqrt{KLn\log n})$. This again is $K$ times the gap-free regret of non-conservative matroid algorithms \citep{kveton2014matroid}.

Finally, the regret of $\cmucb1$ contains two gaps $\Delta^\ast_{e^\ast,\min}$ and $\Delta_{e,\min}$, while the regret of $\cmucb2$ contains an additional gap $\Delta'_{e',\min}$ that is defined for items $e' \in B_0$. The gap $\Delta_{e,\min}$ also appears in the regret of non-conservative matroid algorithms \citep{kveton2014matroid}. The gap $\Delta^\ast_{e^\ast,\min}$ measures the distance of every optimal item to the closest suboptimal item, and is similar to that appearing in top-$K$ best arm identification problems \citep{kalyanakrishnan2012pac}. We believe the $\Delta'_{e',\min}$ gap in the $\cmucb2$ regret bound is not necessary and our analysis can be improved; however note that it only appears for items in $B_0$, which contains $K$ items, and hence its contribution is small. It also doesn't affect the gap-free bound.

%% file: Experiments.tex

\vspace{-7pt}
\section{Experiments}
\label{sec:experiments}
\vspace{-5pt}
We conduct two experiments. In \cref{sec:regret scaling}, we validate that the
regret of $\cmucb$ grows as per our upper bounds in
\cref{sec:analysis}. In \cref{sec:recommender system experiment}, we solve two
recommendation problems using $\cmucb$, and validate that its regret is no higher than $K-1$ times that of a non-conservative matroid bandit algorithm $\omm$ \citep{kveton2014matroid}. $\omm$ violates our conservative constraint multiple times.

\begin{figure*}[t]
  \label{fig:experiments}
  \centering
  \includegraphics[width=1.8in]{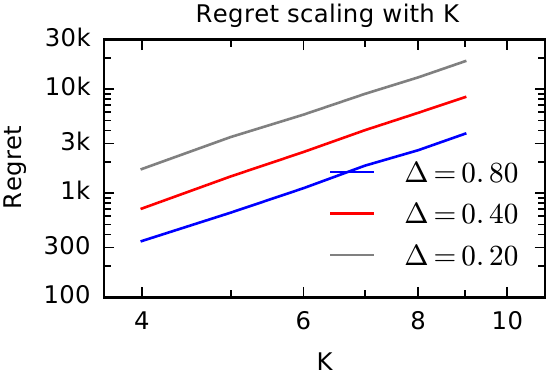}
  \includegraphics[width=1.8in]{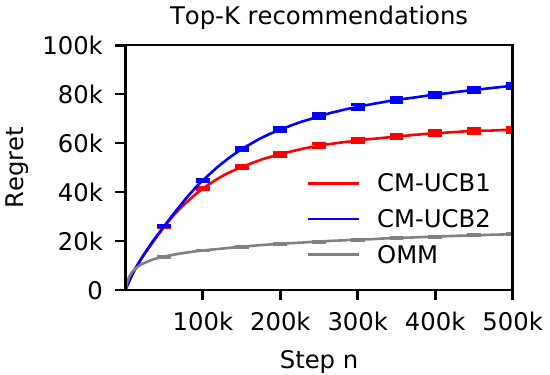}
  \includegraphics[width=1.8in]{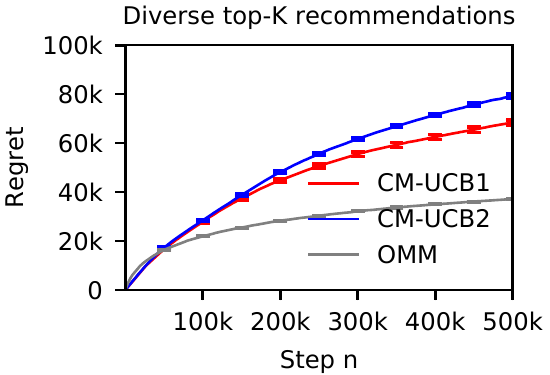} \\
  \hspace{0.35in} (a) \hspace{1.6in} (b) \hspace{1.75in} (c) \vspace{-0.07in}
  \caption{\textbf{a}. The $n$-step regret of $\cmucbone$ in the synthetic problem in \cref{sec:regret scaling} as a function of $K$. \textbf{b}. The regret of $\cmucbone$, $\cmucbtwo$, and $\omm$ in the top-$K$ recommendation problem in \cref{sec:recommender system experiment}. \textbf{c}. The regret of $\cmucbone$, $\cmucbtwo$, and $\omm$ in the diverse top-$K$ recommendation problem in \cref{sec:recommender system experiment}.}
  \vspace{-12pt}
\end{figure*}

  \vspace{-7pt}
\subsection{Regret Scaling}
\label{sec:regret scaling}
  \vspace{-5pt}

The first experiment shows that the regret of $\cmucbone$ grows as suggested by our gap-dependent upper bound in \cref{thm:regret bound cmucb1}. We experiment with uniform matroids of rank $K$ where the ground set is $E = [K^2]$. The $i$-th entry of $\rnd{w}_t$, $\rnd{w}_t(i)$, is an independent Bernoulli variable with mean $\bar{w}(i) = 0.5 (1 - \Delta \1{}{i > K})$ for $\Delta \in (0, 1)$. The baseline set is the last $K$ items in $E$, $B_0 = [K^2] \setminus [K (K - 1)]$. The key property of our class of problems is that the regret of any item in $B_0$ is the same as that of any suboptimal item, and therefore the regret of $\cmucbone$ should be dominated by the gap-dependent term in \cref{thm:regret bound cmucb1}. This term is $O(K^3)$ because $L = K^2$. We vary $K$ and report the $n$-step regret in $100\text{k}$ steps for multiple values of $\Delta$.

\cref{fig:experiments}a shows log-log plots of the regret of $\cmucbone$ as a function of $K$ for three values of $\Delta$. The slopes of the plots are $2.99$ ($\Delta = 0.8$), $2.98$ ($\Delta = 0.4$), and $2.99$ ($\Delta = 0.2$). This means that the regret is cubic in $K$, as suggested by our upper bound.

  \vspace{-7pt}
\subsection{Recommender System Experiment}
\label{sec:recommender system experiment}
  \vspace{-5pt}
  In the second experiment, we apply $\cmucb$ to the two recommendation problems discussed in \cref{sec:conmatbandit}. In each problem, we recommend $K$ most attractive movies out of $L$ subject to a different matroid constraint. We experiment with the \emph{MovieLens} dataset from February 2003 \cite{movielens}, where $6$ thousand users give one million ratings to $4$ thousand movies.

Our learning problems are formulated as follows. The set $E$ are $200$ movies from the MovieLens dataset. The set is partitioned as $E = \bigcup_{i = 1}^{10} E_i$, where $E_i$ are $20$ most popular movies in the $i$-th most popular MovieLens movie genre that are not in $E_1, \dots, E_{i - 1}$. The weight of item $e$ at time $t$, $\rnd{w}_t(e)$, indicates that item $e$ attracts the user at time $t$. We assume that $\rnd{w}_t(e) = 1$ if and only if the user rated item $e$ in our dataset. This indicates that the user watched movie $e$ at some point in time, perhaps because the movie was attractive. The user at time $t$ is drawn randomly from all MovieLens users. The goal of the learning agent is to learn a list of items with the highest expected number of attractive movies on average, subject to a constraint.


We experiment with two constraints. The first problem is a uniform matroid of rank $K = 10$. The optimal solution is the set of $K$ most attractive movies. This setting is also known as \emph{top-$K$ recommendations}. The baseline set $B_0$ are the $11$-th to $20$-th most attractive movies. The second problem is a partition matroid of rank $K = 10$, where the partition is $\{E_i\}_{i = 1}^{10}$. The optimal solution are most attractive movies in each $E_i$. This setting can be viewed as \emph{diverse top-$K$ recommendations}. The baseline set $B_0$ are second most attractive movies in each $E_i$. 

Our results are reported in Figures \ref{fig:experiments}b and \ref{fig:experiments}c. We observe several trends. First, the regret of all algorithms flattens over time, which shows that they learn near-optimal solutions. Second, the regret of $\cmucbtwo$ is higher than that of $\cmucbone$. This is because $\cmucbtwo$ is a variant of $\cmucbone$ that does not know the values of suboptimal items, and therefore needs to estimate them. Both of our algorithms satisfy our conservative constraint in \eqref{eq:conservative constraint} at each time $t$. Third, we observe that $\omm$ achieves the lowest regret. But it also violates our conservative constraints. In Figures \ref{fig:experiments}b and \ref{fig:experiments}c, the numbers of violated constraints are more than $16$ and $158$ thousand, respectively. In the latter problem, this is one violated constraint in every three actions on average. Finally, note that the regret of $\cmucb1$ and $\cmucb2$ is less than $(K-1)$ times ($K=10$) the regret of $\omm$, as predicted by our regret bounds.

%% file: RelatedWork.tex
\vspace{-8pt}
\section{Related Work}
\label{sec:relatedwork}
\vspace{-8pt}
Online learning with matroids was introduced by
\citet{kveton2014matroid}, and also studied by \citet{talebi2016optimal}. However, they do not consider any notion of
conservatism. Our $\cmucb$ algorithm borrows ideas and the {\sc MaxBasis} method from
their algorithm.

Conservatism in online learning was introduced by \citet{wu2016conservative}.
They consider the standard multi-armed bandit problem with no structural
assumption about their actions. 
Their constraint is cumulative, and this allows the learner to take bad actions once in a while, but our instantaneous constraint \eqref{eq:conservative constraint} explicitly forbids this by design. However, note that our setting and algorithm applies to combinatorial action spaces, and hence is less general. 

\citet{kazerouni2017conservative} study conservatism in linear bandits. Their constraint is also cumulative; furthermore
the time complexity of their algorithm grows with time when the rewards of
the basline policy are unknown. $\cmucb$ is efficient because it exploits the matroid structure of the action space.

\citet{bastani2017exploiting} study contextual bandits and propose diversity
assumptions on the environment. 
Intuitively, if contexts vary a lot over time, the environment explores on your
behalf and you need not explore. In our setting, the learner actively explores, albeit in a constrained fashion.

\citet{radlinski2006minimally} propose randomizing the order of presented items
to estimate their true relevance in the presence of item and position biases. While their
algorithm guarantees that the quality of the presented items is unaffected, it
does not learn a better policy. The idea of interleaving has been used to evaluate information retrieval systems and \citet{chapelle2012large} validate its efficacy, but they too do not learn a better policy. Our algorithm learns a better policy, as seen in
our regret plots. While we do not consider item and position biases in this work, we
hope to do so in the future work.

%% file: Conclusions.tex

\vspace{-7pt}
\section{Conclusions}
\label{sec:conclusions}
\vspace{-6pt}
In this paper, we study controlled exploration in combinatorial action spaces using interleaving, and precisely formulate the learning problem in the action space of matroids. Our conservate formulation is more suitable for combinatorial spaces than existing notions of conservatism. We propose an algorithm for solving our problem, $\cmucb$, and prove gap-dependent upper bounds on its regret. $\cmucb$ exploits the idea of interleaving, and hence can evaluate an action without ever taking that action.

We leave open several questions of interest. First, we only study the case of $\alpha = 1 / K$. Our algorithm generalizes to higher values of $\alpha$ in uniform and partition matroids, because they satisfy the property that $\forall\,B_1,B_2\in \mathcal{B}$, there exists a bijection $\sigma_{B_1,B_2}: B_1 \to B_2$ such that $(B_1 \setminus X) \cup \sigma_{B_1,B_2}(X) \in \mathcal{B}$ $\forall X \subseteq B_1$. Matroids that satisfy this property are called \emph{strongly base-orderable}, and one can generalize $\cmucb$ and its analysis to these matroids for higher values of $\alpha$ (see \cref{sec:discussion}). It is not clear how to extend our results beyond $\alpha = 1 / K$ when the matroid is not strongly base-orderable.

Second, we exploit the modularity of our reward function. In general, it may not be possible to build unbiased estimators with interleaving. For e.g., clicks are known to be position-biased, and click models that take this into account have non-linear reward functions \cite{chuklin2015click}. But it may be possible to build biased estimators with the right bias, such that a more attractive item never appears to be less attractive than a less attractive item \cite{zoghi2017online}.

Third, \cref{lem:bijectiveexchange} only guarantees the \emph{existence} of a bijection, but it is not constructive. The construction is straightforward for uniform and partition matroids in our experiments. Fourth, we also leave open the question of a lower bound. Finally, note that our new analysis based on \cref{lem:bijectiveexchange} significantly simplifies the original analysis of OMM in \citet{kveton2014matroid}.

%% file: Appendix.tex
\clearpage
\appendix
\section{Appendix}
\label{sec:appendix}
We define a ``good'' event
\begin{align}
  \mathcal{E}_t =\{\forall \, e \in E: |\bar{w}(e) - \hat{\rnd{w}}_{\rnd{T}_{t-1}(e)}(e)| \leq c_{n,\rnd{T}_{t-1}(e) }\}\,,
  \label{eq:goodevent} 
\end{align}
which states that $\bar{w}(e)$ is inside the high-probability confidence interval around $\hat{\rnd{w}}_{\rnd{T}_{t-1}(e)}(e)$ for all items $e$ at the beginning of time $t$.

\begin{lemma}
\label{lem:failureevent} Let $\mathcal{E}_r$ be the good event in \eqref{eq:goodevent}. Then
\begin{align*}
  \mathbb{P}\left(\bigcup_{r = 1}^{n / K} \bar{\mathcal{E}}_r\right) \leq
  \sum_{r = 1}^{n / K} \E{}{\1{}{\bar{\mathcal{E}}_r}} \leq
  \frac{2 L}{K n}\,.
\end{align*}
\end{lemma}
\begin{proof}
From the definition of our confidence intervals and Hoeffding's inequality \cite{boucheron2013concentration},
\begin{align*}
  \mathbb{P}({|\bar{w}(e) - \hat{\rnd{w}}_s(e)| \geq c_{t,s}}) \leq 2 \exp[-3 \log t]
\end{align*}
for any $e \in E$, $s \in [n]$, and $t \in [n]$. Therefore,
\begin{align*}
  \mathbb{P}\left(\bigcup_{r = 1}^{n / K} \bar{\mathcal{E}}_r\right)
  & \leq \sum_{r = 1}^{n / K} \mathbb{P}(\bar{\mathcal{E}}_r) \\
  & \leq \sum_{r = 1}^{n / K} \sum_{e \in E} \sum_{s = 1}^{r K} \mathbb{P}(|\bar{w}(e) - \hat{\rnd{w}}_s(e)| \geq c_{n, s}) \\ 
  & \leq 2 \sum_{e \in E} \frac{1}{K n}\,.
\end{align*}
This concludes our proof.
\end{proof}

\begin{restatable}{lemma}{LemBaselineConstruction}
\label{lem:baseline construction} Let $A$ be the maximum weight basis with respect to weights $w$. Let $B$ be any basis and let $\rho: A \to B$ be the bijection in \cref{lem:bijectiveexchange}. Then
\begin{align*}
  \forall a \in A: w(a) \geq w(\rho(a))\,.
\end{align*}
\end{restatable}
\begin{proof}
    Fix $a \in A$ and let $b = \rho(a)$. By \cref{lem:bijectiveexchange}, $A^a_b = A \setminus \{a\} \cup \{b\} \in \mathcal{B}$. Now note that $A$ is the maximum weight basis with respect to $w$. Therefore,
\begin{align*}
  w(a) - w(b) =  \sum_{e \in A} w(e) - \sum_{e \in A^a_b} w(e) \geq  0\,.
\end{align*}
This concludes our proof.
\end{proof}

\TheoremCmucbOneAccuracy*
\begin{proof}
    At time $t$, the baseline set $\rnd{B}_t$ is the maximum weight basis with respect to $\rnd{v}_t$. Therefore, by \cref{lem:baseline construction}, there exists a bijection $\rnd{\rho}: \rnd{B}_t \rightarrow B_0$ such that 
\begin{align*}
    \forall b \in \rnd{B}_t: \rnd{v}_t(b) \geq \rnd{v}_t(\rnd{\rho}(b))\,.
\end{align*}
From the definition of $\rnd{v}_t$, $\rnd{v}_t(\rnd{\rho}(b)) = \bar{w}(\rnd{\rho}(b))$ for any $b \in \rnd{B}_t$, and thus
\begin{align*}
    \forall b \in \rnd{B}_t: \rnd{v}_t(b) \geq \bar{w}(\rnd{\rho}(b))\,.
\end{align*}
Now suppose that event $\mathcal{E}_t$ in \eqref{eq:goodevent} happens. Then $\bar{w}(e) \geq \rnd{L}_{t}(e)$ for any $e \in E$, and it follows that
\begin{align*}
    \forall b \in \rnd{B}_t: \bar{w}(b) \geq \bar{w}(\rnd{\rho}(b))\,.
\end{align*}
Since any action at time $t$ contains $K - 1$ items from $\rnd{B}_t$, the constraint in \eqref{eq:conservative constraint} is satisfied when event $\mathcal{E}_t$ happens.

Finally, we prove that $ \mathbb{P}(\cup_t \bar{\mathcal{E}}_t) \leq 2 L / (K n)$ in \cref{lem:failureevent}. Therefore, $\mathbb{P}(\mathcal{E}_t) \geq \mathbb{P}(\cap_t \mathcal{E}_t) \geq 1 -  2 L / (K n)$. This concludes our proof.
\end{proof}

\begin{restatable}{lemma}{LemOptimalDecision}
    For any $e,e^\ast$, if $e \in \rnd{D}_t$ and $e = \rnd{\pi}_t(e^\ast)$, we have that
\begin{align}
    2c_{n,\rnd{T}_{t-1}(e)} \ge \bar{w}(e^\ast) - \bar{w}(e),\qquad \text{and}\qquad
    \rnd{T}_{t-1}(e) \le \frac{6 \log n}{\Delta_{e,e^\ast}^2} \le \frac{6 \log n}{\Delta_{e,\min}^2}, 
    \label{eq:optimal decision}
\end{align}
where $\Delta_{e,\min}$ is defined in \eqref{eq:Deltaemin}. 
    \label{lem:optimal decision}
\end{restatable}

\begin{proof}
Since the decision set $\rnd{D}_t$ is chosen using upper confidence bounds, we have that $\rnd{U}_t(e) \ge \rnd{U}_t(e^\ast)$. This gives us:
$$\bar{w}(e) + 2c_{n,\rnd{T}_{t-1}(e)} \ge \hat{\rnd{w}}_{t-1}(e) + c_{n,\rnd{T}_{t-1}(e)} = \rnd{U}_t(e) \ge \rnd{U}_t(e^\ast) \ge \bar{w}(e^\ast).$$
This implies the first inequality in \eqref{eq:optimal decision}. Substituting the expression for $c_{n,\rnd{T}_{t-1}(e)}$ from \eqref{eq:ucb1} yields the bound on $\rnd{T}_{t-1}(e)$ in \eqref{eq:optimal decision}.
\end{proof}

\begin{restatable}{lemma}{LemOptimalBaseline}
    For any $e^\ast \in A^\ast$, $e \in \rnd{D}_t$, and $e' \in \rnd{B}_t$ such that $e = \rnd{\pi}_t(e^\ast)$ and $e'=\rnd{\sigma}_t(e)$,
    \begin{itemize}
        \item[(a)] If $e \in A^\ast$, then $e=e^\ast$ and 
\begin{align}
    2c_{n,\rnd{T}_{t-1}(e^\ast)} \ge \bar{w}(e^\ast)-\bar{w}(e'), \qquad \text{and}\qquad 
    \rnd{T}_{t-1}(e^\ast) \le \frac{6 \log n}{\Delta_{e',e^\ast}^2} \le \frac{6 \log n}{\Delta_{e^\ast,\min}^{^\ast 2}}, 
    \label{eq:optimal baseline e=e*}
\end{align}
where $\Delta^\ast_{e^\ast, \min}$ is defined in \eqref{eq:Deltaestarmin}.
\item[(b)] If $e \notin A^\ast$,  
\begin{align}
    4c_{n,\rnd{T}_{t-1}(e)} \ge \bar{w}(e^\ast) - \bar{w}(e').
    \label{eq:optimal baseline ci}
\end{align}
    \end{itemize}
    \label{lem:optimal baseline}
\end{restatable}

\begin{proof}
Since the baseline set is selected using lower confidence bounds, we have that $\rnd{L}_t(e') \ge \rnd{L}_t(e)$. This gives us:
$$\bar{w}(e') \ge \rnd{L}_t(e') \ge \rnd{L}_t(e) \ge \bar{w}(e)-2c_{n,\rnd{T}_{t-1}(e)}$$
This implies that 
\begin{align}
    2c_{n,\rnd{T}_{t-1}(e)} \ge \bar{w}(e)-\bar{w}(e').
    \label{eq:decision baseline ci}
\end{align}
\begin{itemize}
    \item[(a)] If $e \in A^\ast$, then since $e = \rnd{\pi}_t(e^\ast)$, we must have that $e \ne e^\ast$. Assume otherwise. Then $A^\ast \setminus \{e^\ast\} \cup \{e\}$ is a basis (by \cref{lem:bijectiveexchange}) of size $(K-1)$, which contradicts the fact that all bases have the same cardinality $K$. Substituting $e=e^\ast$ in \eqref{eq:decision baseline ci} gives the first inequality in \eqref{eq:optimal baseline e=e*}. The $\rnd{T}_{t-1}(e^\ast)$ bound in \eqref{eq:optimal baseline e=e*} follows by substituting the expression of $c_{n,\rnd{T}_{t-1}(e)}$ from \eqref{eq:ucb1}.
    \item[(b)] If $e \notin A^\ast$, note that the confidence interval inequality in \eqref{eq:optimal decision} from \cref{lem:optimal decision} still holds because $e \in \rnd{D}_t$. \eqref{eq:optimal baseline ci} then follows by adding this and \eqref{eq:decision baseline ci}. 
\end{itemize}
\end{proof}

\TheoremRegretCmucbOne*
\begin{proof}
    We first decompose the regret depending on whether the event $\bar{\mathcal{E}} = \bigcup\limits_{t=1}^{n/K} \bar{\mathcal{E}}_t$ happens or not, where $\mathcal{E}_t$ is defined in \eqref{eq:goodevent}.

    Let $\rnd{R}_t$ denote the regret at time $t$. Then, we can decompose the regret of $\cmucb1$ as:
    \begin{align}
        R(n) &= \E{}{\1{}{\bar{\mathcal{E}}} \sum\limits_{t=1}^{n/K} \rnd{R}_t} +
        \E{}{\1{}{\mathcal{E}}\sum\limits_{t=1}^{n/K} \1{}{\rnd{R}_t}}
    \label{eq:regretdecomposition}
    \end{align}
    Let us first analyze the case when $\bar{\mathcal{E}}$ holds. The probability of this event by \cref{lem:failureevent} is
    $\frac{2L}{Kn}$. Since the maximum regret in $n$ steps can be $Kn$, the contribution of the first term is $2L$.

    We assume $\mathcal{E}$ holds in the remaining proof. The expected regret at time $t$ can be written as
\begin{align}
    \E{}{R_t} &= K \sum_{e^\ast \in A^\ast} \bar{w}(e^\ast)    - (K-1) \sum_{e' \in \rnd{B}_t} \bar{w}(e')    - \sum_{e \in \rnd{D}_t} \bar{w}(e) \nonumber \\ 
    &= \left( \sum_{e^\ast \in A^\ast} \bar{w}(e^\ast)  - \sum_{e \in \rnd{D}_t} \bar{w}(e) \right) + (K-1) \left( \sum_{e^\ast \in A^\ast} \bar{w}(e^\ast) - \sum_{e' \in \rnd{B}_t} \bar{w}(e')\right). \label{eq:regret time t decomposition}
\end{align}

Let us first bound the regret due to the first term. When we sum the first term in \eqref{eq:regret time t decomposition} over all times $t$, we get
\begin{align*}
    \sum_{t=1}^{n/K} \left( \sum_{e^\ast \in A^\ast} \bar{w}(e^\ast)  - \sum_{e \in \rnd{D}_t} \bar{w}(e) \right) 
    &\overset{(a)}{\le} \sum_{t=1}^{n/K} \sum_{e \in \rnd{D}_t} 2c_{n,\rnd{T}_{t-1}(e)} 
    \le \sum_{e \in E\setminus A^\ast} \sum_{t=1}^{n/K} 2\sqrt{\frac{1.5\log n}{\rnd{T}_{t-1}(e)}}\1{}{e \in \rnd{D}_t}
\end{align*}
where $(a)$ follows from the first inequality in \eqref{eq:optimal decision} in \cref{lem:optimal decision}. Since a) the counter $\rnd{T}_{t-1}(e)$ increments every time $e$ is played, b) second inequality in Eq. \eqref{eq:optimal decision} holds by \cref{lem:optimal decision}, and 
\begin{align}
    \sum_{s=1}^{m} \frac{1}{\sqrt{s}} \le 1+2\sqrt{m},
    \label{eq:sum of sqrts}
\end{align}
we can bound the regret due to the first term as 
\begin{align}
    \sum_{t=1}^{n/K} \left( \sum_{e^\ast \in A^\ast} \bar{w}(e^\ast)  - \sum_{e \in \rnd{D}_t} \bar{w}(e) \right) &\le \sum_{e \in E\setminus A^\ast} 2\sqrt{1.5 \log n} \left(1+2\sqrt{\frac{6 \log n}{\Delta_{e,\min}^2}} \right) \nonumber \\
    &\le 12 \sum_{e \in E\setminus A^\ast} \frac{1}{\Delta_{e,\min}} \log n + L\sqrt{6 \log n} \label{eq:regret bound first term}
\end{align}

Let us now bound the regret due to the second term in \eqref{eq:regret time t decomposition}. When we sum the second term in \eqref{eq:regret time t decomposition} over all times $t$, we get
\begin{align}
    &(K-1) \sum_{t=1}^{n/K} \left( \sum_{e^\ast \in A^\ast} \bar{w}(e^\ast)   - \sum_{e' \in \rnd{B}_t} \bar{w}(e') \right) \nonumber \\
    \overset{(a)}{\le}\,& (K-1) \left( \sum_{t=1}^{n/K} \sum_{e \in \rnd{D}_t \cap A^\ast} 2 c_{n,\rnd{T}_{t-1}(e)} + \sum_{t=1}^{n/K} \sum_{e \in \rnd{D}_t \setminus A^\ast} 4 c_{n,\rnd{T}_{t-1}(e)} \right) \nonumber \\ 
    =\,& (K-1) \left( \sum_{e \in A^\ast} \sum_{t=1}^{n/K} 2 \sqrt{\frac{1.5 \log n}{\rnd{T}_{t-1}(e)}} \1{}{e \in \rnd{D}_t} + \sum_{e \in E\setminus A^\ast} \sum_{t=1}^{n/K} 4 \sqrt{\frac{1.5 \log n}{\rnd{T}_{t-1}(e)}} \1{}{e \in \rnd{D}_t}\right) \label{eq:second term decomposition}
\end{align}
where $(a)$ follows from \eqref{eq:optimal baseline e=e*} and \eqref{eq:optimal baseline ci} in \cref{lem:optimal baseline}.
We use the $\rnd{T}_{t-1}(e^\ast)$ bound in \eqref{eq:optimal baseline e=e*} to bound the first term, and the $\rnd{T}_{t-1}(e)$ bound in \eqref{eq:optimal decision} to bound the second term in \eqref{eq:second term decomposition}. Then, from the fact that the counter $\rnd{T}_{t-1}(e)$ is incremented every time $e$ is chosen, and \eqref{eq:sum of sqrts}, we can bound the regret due to the second term in \eqref{eq:regret time t decomposition} as
\begin{align}
    &(K-1) \sum_{t=1}^{n/K} \left( \sum_{e^\ast \in A^\ast} \bar{w}(e^\ast)   - \sum_{e' \in \rnd{B}_t} \bar{w}(e') \right) \nonumber  \\ 
    \le\, &(K-1) \left( \sum_{e^\ast \in A^\ast}2 \sqrt{1.5 \log n} \left( 1+2\sqrt{\frac{6\log n}{\Delta_{e^\ast,\min}^{'2}}}\right)     +      \sum_{e \in E\setminus A^\ast} 4\sqrt{1.5\log n} \left(1+2\sqrt{\frac{6 \log n}{\Delta_{e,\min}^2}} \right)\right) \nonumber \\ 
    \le\, &24(K-1) \sum_{e \in E\setminus A^\ast} \frac{1}{\Delta_{e,\min}} \log n + 
    12(K-1) \sum_{e^\ast \in A^\ast} \frac{1}{\Delta^\ast_{e^\ast,\min}} \log n \nonumber \\
    &\qquad + L(K-1)\sqrt{24 \log n} + K(K-1)\sqrt{6 \log n}
    \label{eq:regret bound second term}
\end{align}

Adding \eqref{eq:regret bound first term}, \eqref{eq:regret bound second term}, and the contribution from the failure event $\bar{\mathcal{E}}$ yields the upper bound in the theorem statement.
\end{proof}

\TheoremCmucbTwoAccuracy*
\begin{proof}
    At time $t$, the baseline set $\rnd{B}_t$ is the maximum weight basis with respect to $\rnd{v}_t$. Therefore, by \cref{lem:baseline construction}, there exists a bijection $\rnd{\rho}:\rnd{B}_t \rightarrow B_0$ such that
\begin{align*}
    \forall b \in \rnd{B}_t: \rnd{v}_t(b) \geq \rnd{v}_t(\rnd{\rho}(b))\,.
\end{align*}

Now we consider two cases. First, suppose that $b \in B_0$. Then by \cref{lem:baseline construction}, $b = \rnd{\rho}(b)$, and $\bar{w}(b) \geq \bar{w}(\rnd{\rho}(b))$ from our assumption. Second, suppose that $b \notin B_0$. Then from $\rnd{v}_t(b) = \rnd{L}_{t}(b)$ and  $\rnd{v}_t(\rnd{\rho}(b)) = \rnd{U}_{t}(\rnd{\rho}(b))$, and
\begin{align*}
    \bar{w}(b) \geq  \rnd{L}_{t}(b) \geq \rnd{U}_{t}(\rnd{\rho}(b)) \geq \bar{w}(\rnd{\rho}(b))
\end{align*}
under event $\mathcal{E}_t$. Since any action at time $t$ contains $K - 1$ items from $\rnd{B}_t$, the constraint in \eqref{eq:conservative constraint} is satisfied when event $\mathcal{E}_t$ happens.

Finally, we prove that $ \mathbb{P}(\cup_t \bar{\mathcal{E}}_t) \leq 2 L / (K n)$ in \cref{lem:failureevent}. Therefore, $\mathbb{P}(\mathcal{E}_t) \geq \mathbb{P}(\cap_t \mathcal{E}_t) \geq 1 -  2 L / (K n)$. This concludes our proof.
\end{proof}

\begin{restatable}{lemma}{LemCmucbTwo}
    For any $e^\ast \in A^\ast$, $e \in \rnd{D}_t$, and $e' \in \rnd{B}_t$ such that $e' \in B_0$, $e = \rnd{\pi}_t(e^\ast)$, and $e'=\rnd{\sigma}_t(e)$, 
    \begin{itemize}
        \item[(a)] If $e \in A^\ast$, and $c_{n,\rnd{T}_{t-1}(e')} \le c_{n,\rnd{T}_{t-1}(e)}$, then $e = e^\ast$, and 
\begin{align}
    4c_{n,\rnd{T}_{t-1}(e^\ast)} \ge \bar{w}(e^\ast)-\bar{w}(e'), \qquad \text{and} \qquad
    \rnd{T}_{t-1}(e^\ast) \le \frac{24 \log n}{\Delta_{e',e^\ast}^2}. 
    \label{eq:optimal b0 e=e*}
\end{align}
\item[(b)] If $e \in \rnd{D}_t \setminus A^\ast$ and $c_{n,\rnd{T}_{t-1}(e')} \le c_{n,\rnd{T}_{t-1}(e)}$, then
\begin{align}
    6c_{n,\rnd{T}_{t-1}(e)} \ge \bar{w}(e^\ast) - \bar{w}(e').
    \label{eq:optimal b0 ci}
\end{align}
\item[(c)] If $c_{n,\rnd{T}_{t-1}(e')} > c_{n,\rnd{T}_{t-1}(e)}$, then
\begin{align}
    4c_{n,\rnd{T}_{t-1}(e')} \ge \bar{w}(e) - \bar{w}(e'), \qquad \text{and} \qquad 
    \rnd{T}_{t-1}(e') \le \frac{24 \log n}{\Delta_{e',e}^2} \le \frac{24 \log n}{\Delta_{e',\min}^{'2}}, 
    \label{eq:decision b0 case2}
\end{align}
where $\Delta'_{e',\min}$ is defined in \eqref{eq:Deltaetildemin}.
    \end{itemize}
    \label{lem:cmucb2}
\end{restatable}

\begin{proof}
For items $e' \in B_0 \cap \rnd{B}_t$, we have that $\rnd{U}_t(e') \ge \rnd{L}_t(e)$. This gives us
$$\bar{w}(e') + 2c_{n,\rnd{T}_{t-1}(e')} \ge \rnd{U}_t(e') \ge \rnd{L}_t(e) \ge \bar{w}(e)-2c_{n,\rnd{T}_{t-1}(e)}$$
This implies that 
\begin{align}
    2c_{n,\rnd{T}_{t-1}(e)} + 2c_{n,\rnd{T}_{t-1}(e')} \ge \bar{w}(e)-\bar{w}(e').
    \label{eq:decision b0 ci}
\end{align}
    \begin{itemize}
    \item[(a)] If $e \in A^\ast$, then $e=e^\ast$ by the same argument as in the proof of \cref{lem:optimal baseline}(a). Substituting $e=e^\ast$ in \eqref{eq:decision b0 ci} gives the first inequality in \eqref{eq:optimal b0 e=e*}. Substituting the expression for $c_{n,\rnd{T}_{t-1}(e^\ast)}$ from \eqref{eq:ucb1} gives the second inequality in \eqref{eq:optimal b0 e=e*}.
    \item[(b)] If $e \in \rnd{D}_t \setminus A^\ast$ and $c_{n,\rnd{T}_{t-1}(e')} \le c_{n,\rnd{T}_{t-1}(e)}$, adding the confidence interval inequalities in \eqref{eq:decision b0 ci} and \eqref{eq:optimal decision} gives \eqref{eq:optimal b0 ci}.
    \item[(c)] We assume $\bar{w}(e) > \bar{w}(e')$, because otherwise the regret contribution is bounded by $0$. Then, $c_{n,\rnd{T}_{t-1}(e')} > c_{n,\rnd{T}_{t-1}(e)}$ and \eqref{eq:decision b0 ci} imply the first inequality in \eqref{eq:decision b0 case2}. 
        Substituting the expression for $c_{n, \rnd{T}_{t-1}(e')}$ from \eqref{eq:ucb1} gives the bound on $\rnd{T}_{t-1}(e')$ in \eqref{eq:decision b0 case2}.
    \end{itemize}
\end{proof}

\begin{corollary}
For any $e^\ast \in A^\ast \cap \rnd{D}_t$, and $e' \in \rnd{B}_t$ such that and $e'=\rnd{\sigma}_t(e^\ast)$, if a) $e' \notin B_0$, or b) $e' \in B_0$ and $c_{n,\rnd{T}_{t-1}(e')} \le c_{n,\rnd{T}_{t-1}(e^\ast)}$, we have
\begin{align}
    \rnd{T}_{t-1}(e^\ast) \le \frac{24 \log n}{\Delta_{e^\ast,\min}^{\ast 2}}, 
    \label{eq:Ttestar bound cmucb2}
\end{align}
where $\Delta^\ast_{e^\ast,\min}$ is defined in \eqref{eq:Deltaestarmin}.
    \label{corr:Ttestar cmucb1 cmucb2}
\end{corollary}
\begin{proof}
    The proof follows by taking the maximum of the upper bounds in \eqref{eq:optimal baseline e=e*} and \eqref{eq:optimal b0 e=e*} over all $e'$ that satisfy the conditions of \cref{lem:optimal baseline}(a) or \cref{lem:cmucb2}(a).
\end{proof}

\TheoremRegretCmucbTwo*
\begin{proof}
    Similar to the proof of $\cmucb1$, we use \eqref{eq:regretdecomposition} to break down the regret depending on whether the failure event $\bar{\mathcal{E}} = \bigcup\limits_{t=1}^{n/K} \bar{\mathcal{E}}_t$ holds or not. The contribution from the event $\bar{\mathcal{E}}$ is again bounded by $2L$. 

    We assume $\mathcal{E}$ holds in the remaining proof. We again use \eqref{eq:regret time t decomposition} to decompose the regret, and the bound on the first term from \eqref{eq:regret bound first term} holds. 

The difference in $\cmucb2$ compared to $\cmucb1$ is that while selecting the baseline set $\rnd{B}_t$ in $\cmucb2$, we use upper confidence intervals for items in $B_0$.

We now sum the second term in \eqref{eq:regret time t decomposition} over all times $t$,
\begin{align}
     &(K-1) \sum_{t=1}^{n/K} \left( \sum_{e^\ast \in A^\ast} \bar{w}(e^\ast) - \sum_{e' \in \rnd{B}_t} \bar{w}(e')\right) \nonumber \\
     = &(K-1) \sum_{t=1}^{n/K} \left( \left( \sum_{\stackrel{e^\ast \in A^\ast,}{\rnd{\sigma}_t(\rnd{\pi}_t(e^\ast)) \notin B_0}} \bar{w}(e^\ast) - \sum_{e' \in \rnd{B}_t \setminus B_0} \bar{w}(e') \right) + \left( \sum_{\stackrel{e^\ast \in A^\ast,}{\rnd{\sigma}_t(\rnd{\pi}_t(e^\ast)) \in B_0}} \bar{w}(e^\ast) - \sum_{e' \in \rnd{B}_t \cap B_0} \bar{w}(e') \right) \right) \nonumber \\
     \le &(K-1) \sum_{t=1}^{n/K} \left( \left( \sum_{\stackrel{e \in \rnd{D}_t \cap A^\ast,}{\rnd{\pi}_t(e) \notin B_0}} 2 c_{n,\rnd{T}_{t-1}(e)} + \sum_{\stackrel{e \in \rnd{D}_t \setminus A^\ast,}{\rnd{\pi}_t(e) \notin B_0}} 4 c_{n,\rnd{T}_{t-1}(e)} \right) \right. \nonumber \\
     + &\left. \left( \sum_{\stackrel{e \in A^\ast \cap \rnd{D}_t, \rnd{\pi}_t(e) =e' \in B_0}{c_{n,\rnd{T}_{t-1}(e)} > c_{n,\rnd{T}_{t-1}(e')}}} 4c_{n,\rnd{T}_{t-1}(e)} 
     + \sum_{\stackrel{e \in \rnd{D}_t \setminus A^\ast, \rnd{\pi}_t(e) =e' \in B_0}{c_{n,\rnd{T}_{t-1}(e)} > c_{n,\rnd{T}_{t-1}(e')}}} 6c_{n,\rnd{T}_{t-1}(e)}
 + \sum_{\stackrel{e \in \rnd{D}_t,\rnd{\pi}_t(e) =e' \in B_0}{c_{n,\rnd{T}_{t-1}(e)} > c_{n,\rnd{T}_{t-1}(e')}}} 4c_{n,\rnd{T}_{t-1}(e')} \right) \right) \nonumber \\
     \le &(K-1)\left( \sum_{e^\ast \in A^\ast} \sum_{t=1}^{n/K} 4c_{n,\rnd{T}_{t-1}(e^\ast)} \1{}{e^\ast \in \rnd{D}_t} 
     + \sum_{e \in E\setminus A^\ast} \sum_{t=1}^{n/K} 6c_{n, \rnd{T}_{t-1}(e)} \1{}{e \in \rnd{D}_t} \right. \nonumber \\
     \qquad &\left. + \sum_{e' \in B_0} \sum_{t=1}^{n/K} 4c_{n,\rnd{T}_{t-1}(e')} \1{}{e' \in \rnd{B}_t} \right) \nonumber
 \end{align}
 Similar to the proof of $\cmucb1$, we substitute for the confidence intervals using \eqref{eq:ucb1}. We then bound the first term using \eqref{eq:Ttestar bound cmucb2}, second term using \eqref{eq:optimal decision}, and third term using \eqref{eq:decision b0 case2}.
\begin{align}
     (K-1) &\sum_{t=1}^{n/K} \left( \sum_{e^\ast \in A^\ast} \bar{w}(e^\ast) - \sum_{e' \in \rnd{B}_t} \bar{w}(e')\right) \nonumber \\
    \le (K-1) &\left( \sum_{e^\ast \in A^\ast} \frac{48 \log n}{\Delta^\ast_{e^\ast,\min}} + \sum_{e \in E\setminus A^\ast \setminus B_0} \frac{36 \log n}{\Delta_{e,\min}} + \sum_{e' \in B_0} \frac{48 \log n}{\Delta'_{e',\min}} \right) \nonumber \\ 
    + (K-1)&\left( K\sqrt{24 \log n} + L\sqrt{48 \log n} + K\sqrt{24 \log n} \right) \label{eq:regret bound second term cmucb2}
 \end{align}
 Adding \eqref{eq:regret bound first term}, \eqref{eq:regret bound second term cmucb2} and the contribution from the failure event $\bar{\mathcal{E}}$ yields the upper bound in the theorem statement.
 \end{proof}